\documentclass[10pt,final,conf,twocolumn]{IEEEtran}%
\pdfoutput=1
\usepackage{amsmath}
\usepackage{graphicx}
\usepackage{latexsym}
\usepackage{amssymb}
\usepackage{cite}
\usepackage{color}
\usepackage{multirow}
\usepackage{algorithmic,algorithm}
\usepackage{footnote}
\usepackage[utf8]{inputenc}
\usepackage{comment}
\usepackage{empheq}

\usepackage{mathtools}

\usepackage{amsthm,thmtools}
\declaretheoremstyle[
headfont=\bfseries,
bodyfont=\normalfont,
]{mydefinition}

\newtheorem{theorem}{Theorem}
\newtheorem{lemma}{Lemma}

\newtheorem{definition}{Definition}		
\newtheorem{example}{Example}
\newtheorem{assumption}{Assumption}
\newtheorem{remark}{Remark}

\usepackage{pifont}
\newcommand{\cmark}{\ding{51}}
\newcommand{\xmark}{\ding{55}}

\def\Tr{{\rm Tr}}

\def\be{\begin{equation}}
	\def\ee{\end{equation}}
\def\bea{\begin{eqnarray}}
	\def\eea{\end{eqnarray}}

\def\caD{{\cal D}}
\def\vecy{{\bf y}}
\def\vecx{{\bf x}}
\def\vecz{{\bf z}}
\def\vecu{{\bf u}}
\def\st{{(t)}}
\def\vecn{{\bf n}}
\def\greF{\nabla{F_k({\bf w}^\st)}}
\def\greFD{\nabla{F({\bf w}^\st, \caD_k)}}
\def\greFDd{\nabla{F({\bf w}^\st,\caD_k')}}
\def\ggreF{\nabla{F({\bf w}^\st)}}
\def\hggreF{\widehat{\nabla F}({\bf w}^\st)}
\def\vecw{{\bf w}} 
\def\losF{F_k({\bf w}^\st)}

\usepackage{mathtools}
\DeclarePairedDelimiter{\diagfences}{(}{)}
\newcommand{\diag}{\operatorname{diag}\diagfences}

\usepackage[colorinlistoftodos,prependcaption,backgroundcolor=black!5!white,bordercolor=red]{todonotes}

\usepackage{makecell}

\begin{document}
	\title{Over-the-Air Federated Learning with Privacy Protection via Correlated Additive Perturbations} 
	\author{
		\IEEEauthorblockN{Jialing Liao, Zheng Chen, and Erik G. Larsson} \\
		\IEEEauthorblockA{Department of Electrical Engineering (ISY), Link\"{o}ping University, Link\"{o}ping, Sweden\\
			Email: \{jialing.liao, zheng.chen, erik.g.larsson\}@liu.se} 
		\thanks{This work is supported by Security Link, ELLIIT, and the KAW foundation.}      
	}

	\maketitle
	
	\begin{abstract}
		In this paper, we consider privacy aspects of wireless federated learning (FL) with Over-the-Air (OtA) transmission of gradient updates from multiple users/agents to an edge server.
		OtA FL enables the users to transmit their updates simultaneously with linear processing techniques, which improves resource efficiency. 
		However, this setting is vulnerable to privacy leakage since an adversary node can hear directly the uncoded message. 
		Traditional perturbation-based methods provide privacy protection while sacrificing the training accuracy due to the reduced signal-to-noise ratio. In this work, we aim at minimizing privacy leakage to the adversary and the degradation of model accuracy at the edge server at the same time. More explicitly, spatially correlated perturbations are added to the gradient vectors at the users before transmission. Using the zero-sum property of the correlated perturbations, the side effect of the added perturbation on the aggregated gradients at the edge server can be minimized. In the meanwhile, the added perturbation will not be canceled out at the adversary, which prevents privacy leakage. 
		Theoretical analysis of the perturbation covariance matrix, differential privacy, and model convergence is provided, based on which an optimization problem is formulated to jointly design the covariance matrix and the power scaling factor to balance between privacy protection and convergence performance. Simulation results validate the correlated perturbation approach can provide strong defense ability while guaranteeing high learning accuracy.  
	\end{abstract}
	

	\section{Introduction}\label{sec:pcn-intro}
	
	As one instance of distributed machine learning, federated learning (FL) was developed by Google in 2016, where the clients can train a model collaboratively by exchanging local gradients or parameters instead of raw data \cite{konevcny2016FL}. Research activities on FL over wireless networks have attracted wide attention from various perspectives, such as communication and energy efficiency, privacy and security issues etc \cite{ZhuOtAmag20,HuiPan21EIsurv}. 
	
	Communication efficiency is an important design aspect of wireless FL schemes due to the need of data aggregation over a large set of distributed nodes with limited communication resources. 
	Recently, Over-the-Air (OtA) computation has been applied for model aggregation in wireless FL by exploiting the waveform superposition property of multiple-access channels \cite{ZhuOtA20, sery2020over}. Under OtA FL, edge devices can transmit local gradients or parameters simultaneously, which is more resource-efficient than traditional orthogonal multiple access schemes. 
	
	Despite the extensive research on wireless FL, recent works have shown that traditional FL schemes are still vulnerable to inference attacks on local updates to recover local training data \cite{zhu2019deep, nasr19leakage}. 
	One solution is to reduce information disclosure, which motivates the usage of compression methods such as dropout, selective gradients sharing, and dimensionality reduction \cite{wager2013dropout,shokri2015privacy,fu2019attack}, with the drawbacks of limited defense ability and no accuracy guarantee. Other cryptography technologies, such as secure multi-party computation and homomorphic encryption \cite{wang2016MPC,aono2017privacy} can provide strong privacy guarantees, but yield more computation and communication costs while being hard to implement in practice. Due to easy implementation and high efficiency, perturbation methods such as differential privacy (DP) \cite{dwork2014algorithmic} or CountSketch  matrix \cite{rothchild2020fetchsgd} have been developed. DP technique can effectively quantify the difference in output caused by the change in individual data and reduce information disclosure by adding noise that follows some distributions (e.g., Gaussian, Laplacian, Binomial) \cite{dwork2014algorithmic,agarwal2018cpsgd}. In the context of FL, one can use two DP variants by transmitting perturbed local updates or global updates, i.e., Local DP and Central DP \cite{kairouz14LDP}. 
	However, DP-based methods fail to achieve high learning accuracy and defense ability at the same time due to the reduction of signal-to-noise ratio (SNR), which ultimately limits their application.  
	
	To address this issue, in this paper, we design an efficient perturbation method for OtA FL with strong defense ability without significantly compromising the learning accuracy. Unlike the traditional DP method by adding uncorrelated noise, we add spatially correlated perturbations to local updates at different users/agents. We let the perturbations from different users sum to zero at the edge server such that the learning accuracy is not compromised (with only slightly decreased SNR due to less power for actual data transmission). On the other hand, the perturbations still exist at the adversary due to the misalignment between the intended channel and the eavesdropping channel, which can prevent privacy leakage. 
	
	\subsection {Related Work}
	
	The authors in \cite{Kang20} developed a hybrid privacy-preserving FL scheme by adding perturbations to both local gradients and model updates to defend against inference attacks. 
	In \cite{BurakDeniz2021} the client anonymity in OtA FL was exploited by randomly sampling the devices participating and distributing the perturbation generation across clients to ensure privacy resilience against the failure of clients. Without adversaries but with a curious server, the trade-offs between learning accuracy, privacy, and wireless resources were discussed in \cite{seif2020wireless}. Later on, authors of \cite{Osvaldo21} developed a privacy-preserving FL scheme under orthogonal multiple access (OMA) and OtA, respectively, proving that the inherent anonymity of OtA channels can hide local updates to ensure high privacy. This framework was extended to a reconfigurable intelligent surface (RIS)-enabled OtA FL system by exploiting the channel reconfigurability with RIS \cite{Shi22Ris}. 
	However, the aforementioned approaches reduce privacy leakage at the cost of degrading learning accuracy.  
	
	To this end, authors in \cite{Xue20accuracyLossless} developed a server-aware perturbation method where the server can eliminate the perturbations before aggregation, which requires extra processing and coordination. A more efficient way to balance accuracy and privacy is to guarantee that the inserted perturbations add up to zero. To the best of our knowledge, this strategy has not been explored in wireless FL, although similar ideas exist in the literature of consensus and secure sharing domains. For instance, pair-wise secure keys were exploited in \cite{bonawitz2017} where each user masked its local update via random keys assigned in pairs with opposite signs such that the keys add up to zero. In \cite{he2018privacy}, the perturbation was generated temporally correlated with a geometrically decreasing variance over iterations such that the perturbation adds up to zero after multiple iterations. Compared with these methods, we provide fundamental analysis of general spatially correlated perturbations based on covariance matrix rather than a special case mentioned in \cite{bonawitz2017}. Though the privacy analysis is discussed in the context of the Gaussian mechanism, extensions to other distributions are possible.  
	
	\if 
	A table summarizing related work and ... is provided below. \footnote{Here the privacy attack refers to inference attacks launched by an honest-but-curious adversary (ADV), server (SRV), or user equipment (UE). UL and DL denote uplink and downlink respectively.}

	\setlength{\tabcolsep}{.075mm}{
		\begin{table}[htbp] 
			\begin{center} \caption{Taxonomy of related work on Privacy-Preserving FL}   
				\begin{tabular}{|c|c|c|c|c|c|}   
					\hline Ref. & Net.  & Attacker & OtA & Privacy & Key idea \\  
					\hline \cite{Kang20} & \makecell[c]{   UL\\ DL} & ADV & \xmark & \makecell[c]{Uncorrelated \\ DP} & \makecell[c]{Random user scheduling \\ Effectiveness} \\
					\hline \cite{BurakDeniz2021} &  UL & SRV & \cmark & \makecell[c]{Uncorrelated \\ R{\'e}nyi DP} & \makecell[c]{Subsampling devices and datasets  \\
						Resilience} \\ 
					\hline \cite{seif2020wireless} &  UL & SRV &  \xmark & \makecell[c]{Uncorrelated \\ DP}  & \makecell[c]{Performance analysis  \\  privacy-accuracy-power tradeoffs} \\ 
					\hline \makecell[c]{\cite{Osvaldo21} \\ \cite{Shi22Ris}}& UL & SRV & \cmark & \makecell[c]{Uncorrelated\\ DP}  &  \makecell[c]{Power control \scriptsize{(RIS phase shift design)} \\ privacy-accuracy-power tradeoffs} \\
					\hline
					\cite{Xue20accuracyLossless}  &  \makecell[c]{UL\\ DL} & \makecell[c]{SRV \\ UE} & \xmark &  \makecell[c]{Server aware  \\ random number} & \makecell[c]{Noise elimination \\ accuracy-privacy tradeoff} \\
					\hline
					\makecell[c]{\cite{bonawitz2017} \\ \cite{Avestimehr21}} & UL & ADV & \xmark &  \makecell[c]{Pair-wise \\ secure key} & \makecell[c]{Secure sharing, user scheduling\\ Byzantine-Resilience} \\
					\hline
					\cite{he2018privacy} & UL & ADV &  \xmark & \scriptsize{\makecell[c]{Temporal correlation \\ ($\alpha, \beta$)-privacy}} & \makecell[c]{Privacy-Aware average consensus \\ noise distribution design} \\
					\hline \makecell[c]{Our\\ work} & UL & SRV & \cmark & \makecell[c]{Spatial  \\ correlation, DP}  & \makecell[c]{Noise covariance matrix design \\ power control\\ privacy-accuracy-power tradeoffs}\\  
					\hline   
				\end{tabular}  
			\end{center}   
		\end{table}
	}

	\subsection {Contributions}
	The contributions of this paper are summarized below:
	\begin{itemize}
		\item We develop an OtA FL system achieving trade-offs among resource efficiency, privacy and learning accuracy by exploiting the
		waveform superposition property of the wireless channels, DP mechanism and sum-to-zero correlated perturbations. 
		\item Fundamental analysis of correlated perturbation generation based on covariance matrix is provided, as well as the privacy analysis of DP and the convergence behavior.
		\item A simple approach to generate perturbation covariance matrices and perturbation realizations is developed. 
		\item We formulate an optimization framework that minimizes the optimality gap subject to privacy and transmit power constraints by jointly designing the perturbation covariance matrix and the transmit power scaling factor. 
		\item Simulation results are presented to validate our algorithm and demonstrate the advantages of leveraging correlated noises and OtA to provide efficiency, privacy and accuracy guarantees.
	\end{itemize}
	

	\subsection {Notation} 
	Scalar, vector, matrix and set variables are denoted by $x$, ${\bf x}$, ${\bf X}$ and ${\cal X}$ respectively. For any matrix ${\bf X}$, operators ${\bf X}^T, {\bf X}^H, \Tr({\bf X})$, and $\mathbb{E}({\bf X})$ denote the transpose, Hermitian transpose, trace and statistical expectation, respectively. The notation $\lvert {\cal X} \rvert$ denotes the size of a set ${\cal X}$, and $\lVert {\bf x} \rVert$ denotes the Euclidean norm of a vector ${\bf x}$. 
	\fi

	\section{System Model}\label{sec:sys}
	
	\begin{figure}[t] 
		\centering
		\includegraphics[width=7.cm]{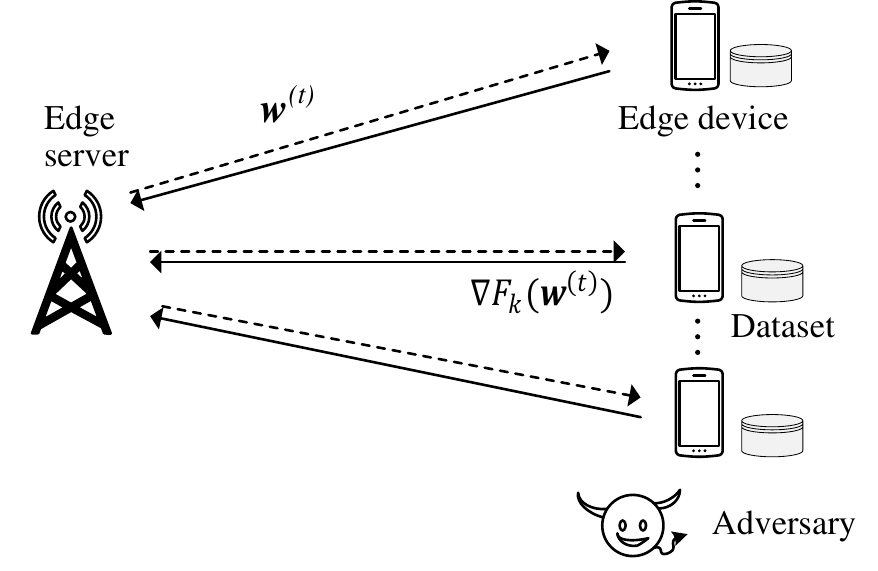}
		\caption{A federated edge learning system with an adversary that can eavesdrop on the local gradients transmitted from the devices.}\label{pcn:sm}  
	\end{figure}
	As shown in Fig.~\ref{pcn:sm}, we consider a wireless FL system where $K$ single-antenna devices intend to transmit gradient updates to an edge server with OtA computation. An adversary is located near one of the users, which intends to overhear the transmissions and infer knowledge about the training data. Each user $k \in \{1,2,\dots K\} \triangleq {\cal K}$ has a local dataset $\caD_k=\{(\vecu_i^k, v_i^k)\}_{i=1}^{D_k}$ composed of  $D_k$ data points, where $\vecu_i^k$ is the $i$-th data point and $v_i^k$ is the corresponding label. The global dataset is then denoted by $\caD=\cup_{k=1}^K \caD_k$ with the total size given by $D_{\text{tot}}=\sum_{k=1}^K D_k$. For sake of brevity, we assume that users have equal-sized datasets, i.e., $D_k=D, \forall k \in \mathcal{K}$. \footnote{The results can be extended to the case where there are datasets with distinct sizes as that does not affect the main structure of the privacy analysis.} The size of the global dataset is thereby given by $D_\text{tot}=KD$. 
	
	Suppose the users jointly train a learning model $\vecw \in \mathbb{R}^d$ by minimizing the global loss function $F(\vecw)$, i.e., $\vecw^* = \arg \min_{\vecw} F(\vecw)$. 
	FL is an iteration process where in every round, each user $k$
	obtains its local gradient vector $\nabla F_k(\vecw)$ using its local dataset. Then, the edge server estimates the global gradient vector by aggregating the received gradient vectors from the users, then update the model parameter vector $\vecw$ to all users. In total, ${\cal T}=[1,2,\dots, T]$ rounds of iteration is considered.
	
	We assume that the edge server and the users are all honest. However, the external adversary is honest-but-curious, which means that it does not attempt to perturb the aggregated gradients but only eavesdrops on the gradient information in order to infer knowledge about the local datasets. Note that in this paper we focus on the uplink transmission of the local gradient updates from the users to the edge server, which belongs to the setting of local DP. The privacy leakage in the downlink transmission of the global model updates to the users is reserved for future work.   
	
	
	\section{Problem Formulation}\label{sec:msuq}
	\subsection{Communication Protocol with Correlated Perturbations} 
	
	Let $\vecx_k^\st$ represent the transmitted signal from the $k$-th user to the edge server during the uplink transmission of local gradient updates in the $t$-th round/iteration. The received signal at the edge server is 
	\be
	\vecy^\st=\sum_{k=1}^{K} h_k^\st \vecx_k^\st+\vecz^\st, \label{pcn:rss}
	\ee
	where $h_k^\st \in {\mathbb C}$ is the channel gain from user $k$. The channel noise $\vecz^\st$ is identically and independently distributed (i.i.d.) in all iterations, and follows ${\cal CN} \sim (0, N_0  \bf{I_d})$. To reduce the information leakage to the adversary, we add perturbations to introduce randomness in the transmitted gradient data. This means that instead of transmitting the true gradient $\greF$, the $k$-th user transmits the following noisy update\footnote{To utilize both the real part and the imaginary part, we split $\greF$ to construct a complex vector with the components $[{\greF}]_i+j[{\greF}]_{i+d/2}, i=1,\dots d/2$. For simplicity, we keep the notations ${\greF}$ and $d$. A de-splitting process is done at the receiver nodes.}
	
	\be
	\vecx_k^\st=\alpha_k^\st \bigg(\greF+\vecn_k^\st \bigg),
	\ee
	where $\alpha_k^\st \in \mathbb{C}$ denotes the transmit scaling factor, given by 
	\be 
	\alpha_k^\st=\sqrt{\eta^\st}/ h_k^\st,
	\ee 	
	and $\eta^\st \in {\mathbb R}^+$ is the common power scaling factor. The transmitted signal consists of two components: the local gradient $\greF$, and the $d \times 1$ artificial noise vector $\vecn_k^\st \in {\mathbb C}^d$. It is assumed that each user has limited power budget $P$, i.e., 
	\be \mathbb{E} [\|{\mathbf {x}}_{k}^\st\|^{2}] \leq P. \ee
	Substituting the transmitted signal $\vecx_k^\st$ into \eqref{pcn:rss}, the received signal at the edge server becomes 
	\be
	\vecy^\st=\sum_{k=1}^{K} \sqrt{\eta^\st}\bigg(\greF+\vecn_k^\st \bigg)+\vecz^\st.
	\ee  
	We utilize spatially correlated perturbations at different users, such that the sum of the added perturbations is $\bf{0}$,  i.e.,     
	\be 
	\sum_{k=1}^K \vecn_k^\st=\bf{0}.   \label{pcn:nc}
	\ee   
	Additionally, it is assumed that the perturbations are independent across different iterations.
	In the following, we describe how the perturbation vectors are generated at the users following a covariance-based design.
	
	We assume that the elements in the $d$-dimensional perturbation vector are independent, so that we can consider each component of $\vecn_k^\st$ independently. We define a $K$-dimensional vector ${\bf v}_i^\st=[\vecn_{1,i}^\st,\vecn_{2,i}^\st \dots \vecn_{K,i}^\st]^T, i=1,2,\dots, d$, which contains the perturbation elements at all $K$ users. Then we can describe the statistical distribution of $\{{\bf v}_i^\st\}$ as i.i.d. ${\cal{CN}}(0, {\bf R}^\st)$, where ${\bf R}^\st= \mathbb{E}[{\bf v}_i^\st {{\bf v}_i^\st}^H]$ (same for all $i$). Let ${\bf u}$ be a $K$-dimensional all-ones vector, it then immediately follows that since ${\bf u}^H {\bf v}_i^\st=0$, we have ${\bf u}^H {\bf R}^\st {\bf u} =0$. Thereby the $K \times K$ covariance matrix ${\bf R}^\st$ should satisfy the following constraints 
	\be \label{pcn:R}
	\sum_k \sum_j {\bf R}_{k,j}^\st=0, \quad
	{\bf R}^\st \succeq 0.
	\ee
	
	The diagonal elements of ${\bf R}^\st$ represent the variances of the perturbations, i.e., ${\bf R}_{k,k}^\st={\mathbb{E}[\lVert \vecn_k \rVert^2]}, \forall k$, while ${\bf R}_{k,j}^\st, \forall j \ne k$ reflects the correlation between $\vecn_k^\st$ and $\vecn_j^\st$. 
	In particular, the uncorrelated perturbation method commonly adopted in the literature corresponds to the special case with ${\bf R}^\st=\diag{{\bf R}_{1,1}^\st, {\bf R}_{2,2}^\st, \dots {\bf R}_{K,K}^\st}$.
	
	For clarification, we present a simple example on how to generate the covariance matrix of the perturbations with power constraints following our correlated perturbation design. 
	\begin{example}
		Consider a case with three users ($K=3$) and the objective is to minimize some convex function $f: \mathbb{R}^{K\times K}\rightarrow\mathbb{R} $ of the covariance matrix ${\bf R}^\st$ subject to the zero-sum perturbation condition and power constraint of each user. The optimization problem can be formulated as 
		\begin{subequations} 
			\begin{align}
				\hspace {1.5pc} & \hspace {-2.7pc} \min _{{\bf R}^\st} \quad f ({\bf R}^\st) \!\!  \\
				&\quad \,\,\hspace {-3.8pc} {\mathrm{ s.t.~}} {\bf R}_{1,1}^\st=4, \quad {\bf R}_{2,2}^\st=4,\quad {\bf R}_{3,3}^\st=4,  \\
				&\quad \,\,\hspace {-3.8pc} \hphantom {\mathrm{s.t.~}} \sum_k \sum_j {\bf R}_{k,j}^\st=0, \quad {\bf R}^\st \succeq 0 \label{eq:cons4}, 
			\end{align}		
		\end{subequations}    
		Since \eqref{eq:cons4} is convex, using this approach one can easily generate covariance matrices by applying different criteria, and test the performance in terms of privacy and learning accuracy numerically. 
	\end{example}
	

	Once the covariance matrix ${\bf R}^\st$ is determined, one can generate random perturbations by multiplying $K \times d$-dimensional white noise ${\bf W}^\st$ with $\big({\bf R}^\st\big)^{\frac{1}{2}}$. Let the white noise matrix be ${\bf W}^\st=[{\bf W}_1^\st, {\bf W}_2^\st, \dots, {\bf W}_K^\st]^T$ where $\{{\bf W}_k^\st\}$ are $d \times 1$ i.i.d. Gaussian noises given by ${\bf W}_k^\st \sim {\cal CN}(0,{\bf I}_d)$. The noise vectors $\{{\bf W}_k^\st\}$ are also uncorrelated. We obtain a $K \times d$ perturbation matrix ${\bf N}^\st$, i.e., 
	\be  
	{\bf N}^\st=\big({\bf R}^\st\big)^{\frac{1}{2}} {\bf W}^\st  \label{eq:pcnN}
	\ee
	where ${\bf N}^\st=[\vecn_1^\st, \vecn_2^\st,\dots,\vecn_K^\st]^T$.
	
	Now we substitute the correlated perturbation generation mechanism into the transmit power constraints of the users, 
	
	\be 
	\mathbb{E} [\|{\mathbf {x}}_{k}^\st\|^{2}]=\frac{\eta^\st}{|h_k^\st|^2} \left [(G_{k}^\st)^{2}+d {\bf R}_{k,k}^\st \right] \leq P, ~\forall k,t, \label{pcn: powerc}
	\ee
	where $G_{k}^\st$ is an upper bound of the norm of local gradient for user $k$, i.e., $\|\greF \| \le G_{k}^\st$. 
	
	
	Based on the zero-sum correlated perturbations given in \eqref{pcn:nc}, the received signal at the edge server can be written as 
	\be \label{recevedsig}
	\vecy^\st=\sum_{k=1}^{K} \sqrt{\eta^\st} \greF+\vecz^\st.
	\ee
	Note that the covariance matrix of the generated perturbations will affect the power scaling factor $\eta^\st$, which in turn affects the received SNR at the edge server.

	Next, we analyze the impact of the correlated perturbations at the adversary. Let $g_k^\st \in {\mathbb C}$ be the channel gain between user $k$ and the adversary. We define the corresponding effective channel gain as 
	\be
	\rho_k^\st=g_k^\st/h_k^\st,
	\ee
	which quantities the misalignment between the channels from each user $k$ to the adversary and to the server.
	The received signal at the adversary is    
	\begin{align}\label{Siga}
		\vecy_a^\st=&\sum_{k=1}^{K} g_k^\st \vecx_k^\st+\vecz_a^\st  \nonumber\\
		=&\sum_{k=1}^{K} \sqrt{\eta^\st} \rho_k^\st \big(\greF+\vecn_k^\st \big) +\vecz_a^\st, \nonumber\\
		=& \sqrt{\eta^\st} \sum_{k=1}^{K} \rho_k^\st \greF \nonumber\\
		&+\sqrt{\eta^\st} \left( \big({\boldsymbol \rho}^\st\big)^T \big({\bf R}^\st\big)^{\frac{1}{2}} {\bf W}^\st \right)^T+\vecz_a^\st, 
	\end{align}
	where the channel noise $\vecz_a^\st$ follows i.i.d. ${\cal CN} (0, N_a \bf{I_d})$ with the variance $N_a$, and ${\boldsymbol \rho}^\st=[\rho_1^\st, \rho_2^\st, \dots, \rho_K^\st]^T$. We define the total effective noise at the adversary as 
	\be 
	{\bf r}^\st=
	\sqrt{\eta^\st} \left( \big({\boldsymbol \rho}^\st\big)^T \big({\bf R}^\st\big)^{\frac{1}{2}} {\bf W}^\st \right)^T+\vecz_a^\st.  \label{pcn:r}
	\ee
	Since both components of ${\bf r}^\st$ are Gaussian, we have ${\bf r}^\st \sim {\cal CN}(0, (m^\st)^2 {\bf I_d})$,   
	where the variance of the effective noise per element is 
	\begin{multline} \label{eq:effN}
		(m^\st)^2=\frac{\eta^\st}{d} \mathbb{E}{\bigg[\big \lVert {({\boldsymbol \rho}^\st)}^T \big({\bf R}^\st \big)^{\frac{1}{2}} {\bf W}^\st \big \rVert^2 \bigg]}+N_a.
	\end{multline}
	
	As is shown in \mbox{\eqref{Siga}-\eqref{eq:effN}}, the adversary receives perturbed gradients, due to the amplitude and phase misalignment between the intended channel and the eavesdropping channel. As a result, the added correlated perturbations that add up to zero at the edge server will not cancel out at the adversary. The impact of the perturbations injected by different users is subject to the dual relation between the perturbations and the effective channel gains $\rho_k^\st$. Consequently, the generation of correlated perturbations should take into account the effective channel gains $\rho_k^\st$. 
	
	The advantage of adding correlated perturbations can be interpreted from the perspective of signal-to-noise ratio (SNR) or signal-to-interference-plus-noise ratio (SINR). Generally, a higher SNR at the edge server yields higher learning accuracy while a smaller SINR at the adversary implies better privacy protection. 
	\begin{remark}\label{pcn:sinr}
		At the edge server, the SNR of the aggregated signals without perturbations, with uncorrelated perturbations and correlated perturbations are
		\begin{align*}
			\text{SNR}_s^\st=\bigg[\underbrace{\frac{\eta_{\text{nom.}}^\st  P_s^\st}{d N_0}}_{nominal},~ \underbrace{\frac{\eta_{\text{pert.}}^\st  P_s^\st}{d (\eta^\st {\bf R}_{k,k}^\st+N_0)}}_{uncorrelated},~ \underbrace{\frac{\eta_{\text{pert.}}^\st  P_s^\st}{d N_0}}_{correlated}\bigg]
		\end{align*}
		where $P_s^\st \triangleq \sum_{k=1}^{K} \|\greF \|^2$. Based on the transmit power constraint \eqref{pcn: powerc}, the power scaling factors for the three different perturbation approaches are given by: \footnote{The power scaling factor will be further optimized jointly with covariance matrix ${\bf R}$ with both power constraints and privacy constraint in section IV\&V. The optimal power scaling and covariance matrix can thereby be distinct with different perturbation approaches. }

		\begin{subequations}
			\begin{empheq}[left=\empheqlbrace]{align*}
				\eta_{\text{nom.}}^\st &= P \min_k \frac{|h_k^\st|^2}{(G_{k}^\st)^{2}},\\
				\eta_{\text{pert.}}^\st &= P \min_k \frac{|h_k^\st|^2}{(G_{k}^\st)^{2}+d {\bf R}_{k,k}^\st}. 
			\end{empheq}
		\end{subequations}
		
		Similarly, we obtain the SINR at the adversary for the three different perturbation methods as 
		\begin{align*}
			\!\!\text{SINR}_a^\st\!\!=\!\! \bigg[\underbrace{\frac{\eta_{\text{nom.}}^\st  P_a^\st}{d N_a}}_{nominal},  \underbrace{\frac{\eta_{\text{pert.}}^\st  P_a^\st}{d\! \left(\!\eta_{\text{pert.}}^\st\sum_k |\rho_k^\st|^2{\bf R}_{k,k}^\st\!\!+ \!N_a\!\right)}}_{uncorrelated},  \underbrace{\frac{\eta_{\text{pert.}}^\st P_a^\st}{ d(m^\st)^2}}_{correlated}\!\bigg]
		\end{align*}
		where $P_a^\st \triangleq \sum_{k=1}^{K}  |\rho_k^\st|^2 \|\greF \|^2$ and $(m^\st)^2$ is given in \eqref{eq:effN}. Here, the SINR is defined as the ratio between the power of the desired signal 
		and the power of the total effective noise including perturbations and receiver noise. 
		
		It can be observed that the received SNR at the edge server with correlated perturbations is slightly smaller than that of non-perturbation case only due to the power cost for transmitting the perturbations. This states that adding correlated perturbations does not significantly affect the learning accuracy. In contrast, with uncorrelated perturbations, there is an apparent degradation of SNR due to the aggregated perturbations in the received signal. Compared with the non-perturbation case, the SINR at the adversary is smaller with both uncorrelated perturbations and correlated perturbations due to the effective noise and smaller power scaling factors. This indicates the advantage of adding perturbations in terms of privacy protection. To conclude, the correlated perturbation approach provides training accuracy and privacy guarantee at the same time.  
	\end{remark}

	\subsection{Learning Protocol} 
	
	In the $t$-th round, the local gradient $\greF$ is computed based on the local dataset $\caD_k$, and the current model parameter vector $\vecw^\st$. The local loss function is given by 
	
	\be
	\losF= \sum_{({\boldsymbol \mu}, \nu) \in \caD_k} f(\vecw^\st,{\boldsymbol \mu}, \nu).
	\ee
	Here, $f(\cdot)$ is the loss function quantifying the prediction error based on the training sample $\boldsymbol{\mu}$ with respect to the label $\nu$. The local gradient is thus obtained as 
	\be
	\hspace {-.4pc} \greF= \sum_{({\boldsymbol \mu}, \nu) \in \caD_k} \nabla f(\vecw^\st,{\boldsymbol \mu}, \nu).
	\ee
	Assuming error-free uplink transmission, the aggregated gradient vector at the edge server is
	\be
	\ggreF= \frac{1}{K} \sum_{k=1}^{K} \greF.
	\ee
	However, due to random fading and noise in wireless channels, the edge server can only obtain an estimated global gradient $\hggreF$, and then update the model parameter vector with a proper step-length $\lambda$ as  
	\be \label{prfl:w}
	\vecw^{(t+1)}=\vecw^\st-\lambda \hggreF. 
	\ee
	
	We assume that the edge server knows the power scaling factor $\eta^\st$ at the users. Then, in the $t$-th round, based on the received signal $\vecy^\st$, the edge server can obtain the estimated global gradient by
	\begin{multline}
		\hggreF=\frac{1}{K}\left(\sqrt{\eta^\st}\right)^{-1} \vecy^\st \\
		\hspace {-.6pc}=\frac{1}{K}\sum_{k=1}^{K} \greF+ \frac{1}{K}\big(\sqrt{\eta^\st}\big)^{-1} \vecz^\st.
	\end{multline}

	
	\subsection{Privacy Analysis}
	The privacy level at the adversary is measured with differential privacy (DP). In the following, we provide some basic definitions and privacy analysis under the OtA FL setting.
	
	\begin{definition}[Differential Privacy \cite{dwork2014algorithmic}]
		DP quantifies how much two neighboring datasets can be distinguished by observing the output (received signal) $\vecy$. Let $\caD$ and $\caD'$ be two neighboring datasets that differ only in one sample, i.e., $\lVert \caD-\caD' \rVert_1=1$. The differential privacy loss corresponding to the log-likehood ratio of events $\vecy|\caD$ and $\vecy|\caD'$ is  
		$${\cal L}_{\caD,\caD'}(\vecy)= \ln \frac{\mathrm P(\vecy|\caD)}{\mathrm P(\vecy|\caD')}.$$ 
		The $(\epsilon, \delta)$-differential privacy is thereby achievable under condition that the absolute value of the DP loss is less than a small value $\epsilon$ with probability higher than $1-\delta$ where $\epsilon \ge 0, \delta \in [0,1]$, i.e., 
		$$\mathrm P(\lvert{\cal L}_{\caD,\caD'}(\vecy)\rvert \le \epsilon) \ge  1-\delta.$$ 
	\end{definition}
	
	DP loss is measured via the probability of observing an output that occurs given a dataset $\caD$, and the probability of seeing the same value given a neighboring dataset $\caD'$, where the probability space is some randomized mechanism. The aim of DP is to guarantee that the distribution of the output given two different inputs does not change too much. Smaller parameters $(\epsilon, \delta)$ imply higher privacy level of the randomized mechanism. Pure $\epsilon$-DP is achieved if $\delta=0$. 
	
	\begin{definition}[Gaussian Mechanism \cite{dwork2014algorithmic}]
		Let $f(\caD)$ be a function in terms of an input $\caD$ subject to $(\epsilon, \delta)$-DP. Suppose a user wants to release function $f(\caD)$, the Gaussian mechanism ${\cal M}$ with variance $\sigma^2$ is then defined as:	
		\be \label{pcn:GM}
		{\cal M}(\caD) \triangleq f(\caD) + \mathcal{N}\left({0,\sigma^2{\bf{I}}} \right). 
		\ee	
	\end{definition}

	\begin{definition}[Sensitivity \cite{dwork2014algorithmic}]
		The $l_2$-sensitivity of function $f$ is denoted by $\Delta_f$, i.e., ${\Delta}_f=\max_{\caD, \caD'} |{f(\caD)-f(\caD')}|_2$.
	\end{definition}
	
	Intuitively, $\Delta_f$ captures the maximum possible change in the output caused by the change in a data point, and thereby gives an upper bound on how much perturbation should be added to hide the change of the single record. The absolute value of DP loss under the Gaussian mechanism \eqref{pcn:GM} is 
	\be \label{pcn:GMDP}
	|{\cal L}_{\caD,\caD'}(y)| = \left|\ln  \frac{\mathrm P\left(y-f(\caD)\right)}{\mathrm P\big(y-f(\caD')\big)}\right|
	\le \left|\ln \frac{\exp{\left(\frac{-x^2}{2\sigma^2}\right)}}{\exp \left( \frac{-(x+\Delta_f)^2}{2\sigma^2}\right)}\right|  \notag
	\ee 
	where $y$ is a possible output and $x=y-f(\caD)$. 
	
	Now we interpret the $(\epsilon, \delta)$-DP principle at the adversary. First, with distributed data, the notion of neighboring global datasets implies that only one local dataset will be different in one sample, i.e., $|\caD'_l-\caD_l|_1=1, \caD'_k=\caD_k, k \ne l$ where $\caD=\cup_{k=1}^K{\caD_k}$ and $\caD'=\cup_{k=1}^K{\caD'_k}$. Then the composition theorem of DP  is applied to measure the privacy level after multiple iterations \cite{dwork2014algorithmic}. Let the received signals during successive $T$ iterations be  $\vecy_a=\{\vecy_a^{(t)}\}_{t=1}^T$. The corresponding DP loss after $T$ rounds of iterations is given by  
	\be
	{\cal L}_{\caD,\caD'}(\vecy_a)=\ln \prod_{t=1}^T \frac{\mathrm P(\vecy_a^\st|\vecy_a^{(1)}, \dots,\vecy_a^{(t-1)}, \caD)}{\mathrm P(\vecy_a^\st|\vecy_a^{(1)}, \dots,\vecy_a^{(t-1)}, \caD')}. \label{pcn:dploss}
	\ee   
	
	Since the randomness comes from the perturbation mechanism while the gradients are deterministic, the probability density profile (PDF) of the effective noise can be utilized to quantify the difference between the outputs of neighboring datasets. 
	
	Let ${\bf v}^\st$ be the difference between the desired signal w.r.t. two neighboring global datasets $\caD$ and $\caD'$, i.e.,
	\begin{multline*}
		\!\!\!\!{\bf v}^\st\!=\!\sum_{k=1}^{K} \! \sqrt{\eta^\st} \rho_k^\st  \left(\!\greFD\!-\! \greFDd \right) \\ 
		=\sqrt{\eta^\st} \rho_l^\st  \!\!\left(\sum_{({\boldsymbol \mu}, \nu) \in \caD_l'} \!\!f(\vecw^\st,{\boldsymbol \mu}, \nu)- \!\!\sum_{({\boldsymbol \mu}, \nu) \in \caD_l} \!\!f(\vecw^\st,{\boldsymbol \mu}, \nu) \right).    
	\end{multline*} 
	\if
	Recalling the upper bound of the loss function and the definition of neighboring datasets, it holds that  
	\begin{multline}
		\!\!\!{\bf v}^\st=\sqrt{\eta^\st} \rho_l^\st \\  \times \left(\sum_{({\boldsymbol \mu}, \nu) \in \caD_l'} f(\vecw^\st,{\boldsymbol \mu}, \nu)-\right. 
		\left. \sum_{({\boldsymbol \mu}, \nu) \in \caD_l} f(\vecw^\st,{\boldsymbol \mu}, \nu) \right) 
	\end{multline}
	\fi 
	To obtain a bound on ${\bf v}^\st$ for the sensitivity analysis, we assume that the norm of the sample-wise loss function is upper bounded as follows \cite{Osvaldo21,Shi22Ris}: 
	\begin{assumption} \label{pcn:assump1}
		\textbf{(Bounded sample-wise gradient)}: The norm of the sample-wise gradient at any iteration defined as $\nabla f(\vecw^\st,{\boldsymbol \mu}, \nu)$ is bounded by a constant value $\gamma^\st$, i.e., 	
		\be
		\lVert \nabla f(\vecw^\st,{\boldsymbol \mu}, \nu) \rVert \le \gamma^\st.  
		\ee
		Based on triangle inequality, this assumption indicates that there will always be a constant $G_k^\st \le D_k \gamma^\st$ satisfying $\|\greF \| \le G_{k}^\st$.  
	\end{assumption}

	From the definition of sensitivity, we define $\Delta^\st$ as the maximum distance between the norms of the desired signals w.r.t. all possible pairs of neighboring datasets $\{{\cal D},{\cal D'}\}$, i.e.,
	\be
	\Delta^\st=\max_{{\cal D},{\cal D'}} \lVert {\bf v}^\st \rVert \le 2 \left(\gamma^\st \sqrt{\eta^\st} \rho_{max}^\st\right),  
	\ee
	where we let $\rho_{max}^\st=\max_k \big|\rho_k^\st\big|$. 
	Then, we obtain the privacy constraint of our considered model in the following theorem. 	
	\begin{theorem}\label{Theo:DPadv}
		The considered OtA FL system with proposed correlated perturbation mechanism is $(\epsilon, \delta)$-differential private if the following condition holds
		\begin{multline} \label{DPyaf}
			\sum_{t=1}^{T} \left(\frac{2   \gamma^\st}{m^\st}  \sqrt{\eta^\st} \rho_{max}^\st\right)^2  \\   < \left(\sqrt{\epsilon+\left(C^{-1}\big(1/\delta \big)\right)^2}-C^{-1}\big(1/\delta \big) \right)^2
			\triangleq {\cal R}_{dp}(\epsilon, \delta)
		\end{multline}
		The function $C(x)$ defined as $C(x)=\sqrt{\pi} x \exp(x^2)$ is introduced to simplify the expression. 
	\end{theorem}
	
	\begin{proof}
		See Appendix.
	\end{proof}
	
	Theorem \ref{Theo:DPadv} states that both the power scaling factor $\eta^\st$ and the variance of effective noise $(m^\st)^2$ contribute to privacy protection. In general, smaller $\eta^\st$ and higher $(m^\st)^2$ result in higher privacy level. The effective noise contains two parts: the perturbations and the channel noise, where the first depends on the scaling factor, the effective channel gain $\boldsymbol{\rho}^\st$ and the correlation matrix ${\bf R}^\st$. This result is in line with the discussions provided in Remark \ref{pcn:sinr}.
	
	\subsection{Convergence Performance}
	In addition to privacy protection, model accuracy is another important aspect of our proposed design. 
	We use the optimality gap between the expectation of the global loss function after $T$ rounds of gradient decent and the optimal loss function $F^*$ as the metric to quantify the convergence performance. Here the expectation is taken over the randomness of the additive channel noise. 
	
	To derive the upper bound of the expected optimality gap, the following assumptions on gradients which are frequently used in the literature \cite{Osvaldo21,Shi22Ris, karimi2016linear}, are introduced below: 
	\begin{assumption} \label{pcn:assump2} 
		\textbf{(Smoothness)}: The global loss function $F(\vecw)$ is smooth and continuously differentiable with Lipschitz continuous gradient $\nabla F(\vecw)$. There exists a constant $L>0$, i.e.,
		\begin{align} \left \|{\nabla F({\mathbf {w}})-\nabla F({\mathbf {w}}') }\right \|\leq L \left \|{ {\mathbf {w}}- {\mathbf {w}}' }\right \|, \forall {\mathbf {w}}, {\mathbf {w}}' \in \mathbb {R}^{d}, \end{align}
		which implies that for all ${\mathbf {w}}, {\mathbf {w}}'\in \mathbb {R}^{d}$, it holds that 
		\begin{align}&\hspace {-.5pc}F({\mathbf {w}}')\leq F({\mathbf {w}})+ \nabla F({\mathbf {w}})^{\sf T}({\mathbf {w}}'- {\mathbf {w}})+\frac {L}{2} \left \|{ {\mathbf {w}}- {\mathbf {w}}' }\right \|^{2}. \end{align} 
		
		Assumption \ref{pcn:assump2} guarantees that the gradient of the loss function would not change arbitrarily quickly w.r.t. the parameter vector. 
	\end{assumption}
	\begin{assumption} \label{pcn:assump3} 
		\textbf{(Polyak-Lojasiewicz Inequality)}: 
		In Polyak-Lojasiewicz (PL) condition, it holds for some $\mu>0$ that
		\be \left \|{\nabla F({\mathbf {w}}) }\right \|^{2}\geq 2\mu \left [{F({\mathbf {w}})-F^{*} }\right], \forall {\mathbf {w}} \in \mathbb {R}^{d}, 
		\ee
		where $F^*$ is the optimal function value of $F(\vecw)$. 
		
		Assumption \ref{pcn:assump3} is more general than the standard assumption of strong convexity \cite{karimi2016linear}. 
	\end{assumption}

	With correlated perturbation mechanism, the received signal at the edge server only contains the desired signal and the channel noise. According to \cite{Osvaldo21, Shi22Ris}, the expected optimality gap after $T$ iterations with learning rate fixed at $\lambda=1/L$ and the assumptions mentioned previously, is upper bounded by 
	\begin{align}&\hspace {-.5pc} \mathbb {E} \left [F \left ({{\mathbf {w}}^{(T+1)} }\right)\right]-F^{*} \leq \left (1-\frac{\mu }{L}\right)^{T} \left [{F \left ({{\mathbf {w}}^{(1)} }\right)-F^{*}}\right] \notag \\  &\qquad\displaystyle { +\,\frac {d}{2L(KD)^{2}} \sum _{t=1}^{T}  \left ({1-\frac{\mu }{L}}\right)^{T-t} \frac{N_{0}}{\eta^{(t)}}}. \label{eq:Convf}  \end{align}	
	As shown in \eqref{eq:Convf}, the upper bound of the expected optimality gap is independent of the perturbations due to the zero-sum property of our perturbation design. This bound is subject to some given constants and the power scaling factor $\eta^\st$ which is our design parameter. Neglecting the constant terms, we focus on minimizing the controllable term $\sum_{t=1}^{T} \left(1-\frac{\mu}{L}\right)^{-t}/\eta^\st$ in the following section.

	\section{System Optimization}\label{pcnsec:opt}
	We aim at developing a power control and perturbation correlation algorithm, which determines the scaling factor $\eta^\st$ and the covariance matrix of the correlated perturbations ${\bf R}^\st$ that minimize the optimility gap while satisfying the privacy constraint, the transmitted power budget and the perturbation correlation conditions over $T$ communication rounds. 
	Before presenting the optimization problem, we reformulate the the variance of the effective noise $(m^\st)^2$ as follows
	\begin{align*}   
		&\mathbb{E}{\left[\big  \lVert {({\boldsymbol \rho}^\st)}^T \big({\bf R}^\st \big)^{\frac{1}{2}} {\bf W}^\st  \big \rVert^2 \right]} \\
		&=\mathbb{E}{\bigg[\big({\boldsymbol \rho}^\st\big)^T \big({\bf R}^\st\big)^{\frac{1}{2}} {\bf W}^\st \big({\bf W}^\st\big)^H {\big({\bf R}^\st\big)^{\frac{1}{2}}}^H \big({\boldsymbol \rho}^\st\big)^*\bigg]}\\
		&=\big({\boldsymbol \rho}^\st\big)^T\big({\bf R}^\st\big)^{\frac{1}{2}}\left(\mathbb{E}{\big[{\bf W}^\st \big({\bf W}^\st\big)^H\big]}\right) {\big({\bf R}^\st\big)^{\frac{1}{2}}}^H \big({\boldsymbol \rho}^\st\big)^*\\
		&=\big({\boldsymbol \rho}^\st\big)^T\big({\bf R}^\st\big)^{\frac{1}{2}} d~{\bf I_K} {\big({\bf R}^\st\big)^{\frac{1}{2}}}^H \big({\boldsymbol \rho}^\st\big)^*\\
		&=d \big({\boldsymbol \rho}^\st\big)^T {\bf R}^\st \big({\boldsymbol \rho}^\st\big)^*. 
	\end{align*}
	Then we get 
	\be
	(m^\st)^2=\eta^\st \big({\boldsymbol \rho}^\st\big)^T {\bf R}^\st \big({\boldsymbol \rho}^\st\big)^*+N_a.
	\ee
	Note that in practical causal settings, future channels and gradient information are unknown. We thus apply static privacy budget allocation over $T$ rounds such that the long-term privacy constraint is separated into $T$ independent privacy constraints.
	Let the privacy budget allocation be 
	\be 
	{\cal R}_{dp}^\st(\epsilon, \delta)=\phi^\st {\cal R}_{dp}(\epsilon, \delta),~\forall t \in {\cal T},
	\ee
	where $\sum_t \phi^\st=1, 0<\phi^\st<1, \forall t$. The coefficients $\phi^\st$ can be generated assuming identical or random privacy allocation. 
	In this case with per-slot constraint, the objective function becomes $\left(1-\frac{\mu}{L}\right)^{-t}/\eta^\st$, which is the controllable term of the optimality gap given in \eqref{eq:Convf}, in the $t$-th round. 
	
	The optimization problem in the $t$-th learning slot is  
	\begin{subequations}\label{pcn:opt0}
		\begin{align}
			{\bf P0:} \hspace {1.5pc} & \hspace {-1.2pc}\min _{{\bf R}^\st,~\eta^\st} \quad \frac{\left(1-\frac{\mu}{L}\right)^{-t}}{\eta^\st} \!\! \label{pcn:opt0-0}\\
			&\quad \,\,\hspace {-3.8pc} {\mathrm{ s.t.}}~\frac{\eta^\st(\gamma^\st \rho_{max}^\st)^2}{ \eta^\st\big({\boldsymbol \rho}^\st\big)^T {\bf R}^\st \big({\boldsymbol \rho}^\st\big)^*+N_a}  \le \frac{{\cal R}_{dp}^\st(\epsilon, \delta)}{4}, \!\! \label{pcn:opt0-1}\\
			&\quad \,\,\hspace {-3.8pc}\hphantom {\mathrm{ s.t.~}} \eta^\st \left [(G_{k}^\st)^{2}+d {\bf R}_{k,k}^\st \right] \leq |h_k^\st|^2 P, \quad \forall k, \!\! \label{pcn:opt0-2}\\ &\quad \,\,\hspace {-3.8pc}\hphantom {\mathrm{s.t.~}} \sum_k \sum_i {\bf R}_{k,i}^\st=0,\label{pcn:opt0-3}\\
			&\quad \,\,\hspace {-3.8pc}\hphantom {\mathrm{s.t. ~}} {\bf R}^\st \succeq 0,\label{pcn:opt0-4}\\
			&\quad \,\,\hspace {-3.8pc}\hphantom {\mathrm{s.t. ~}} \eta^\st > 0. \label{pcn:opt0-5}		
		\end{align}		
	\end{subequations}
	
	It can be observed that {\bf P0} is linear in ${\bf R}^\st$ and $1/\eta^\st$ with positive semi-definite constraint on ${\bf R}^\st$. By change of variables, i.e. letting $b^\st=1/\eta^\st$, {\bf P0} can be reformulated into a convex problem {\bf P1}, and be solved using existing numerical solver, e.g., CVX \cite{cvx}. 
	\begin{subequations}\label{pcn:opt5}   
		\begin{align*}
			{\bf P1:} \hspace {1.5pc} & \hspace {-1.2pc}\min _{{\bf R}^\st,~b^\st} \quad  \left(1-\frac{\mu}{L}\right)^{-t} b^\st \!\! \\
			&\quad \,\,\hspace {-4. pc} {\mathrm{ s.t.}}~(\gamma^\st \rho_{max}^\st)^2  \le \! \frac{{\cal R}_{dp}^\st(\epsilon, \delta)}{4} \!\left(\!\big({\boldsymbol \rho}^\st\big)^T {\bf R}^\st \big({\boldsymbol \rho}^\st\big)^* 
			+ N_a b^\st\!\right), \!\! \\
			&\quad \,\,\hspace {-3.8pc}\hphantom {\mathrm{ s.t.~}} (G_{k}^\st)^{2}+d {\bf R}_{k,k}^\st  \leq b^\st |h_k^\st|^2 P , \quad \forall k, \!\! \\
			&\quad \,\,\hspace {-3.8pc}\hphantom {\mathrm{s.t. ~}} \mbox{\eqref{pcn:opt0-3}-\eqref{pcn:opt0-4}},~b^\st > 0. 
		\end{align*}		
	\end{subequations}
	
	\if
	The convexity of \eqref{pcn:opt0-1} is guaranteed if we can prove that $1/\Tr({\bf A}{\bf R}^\st)$ where ${\bf R}^\st \succeq 0$ and ${\bf A}$ is a deterministic positive semi-definite matrix. Lemma \ref{lem:InvTr} shows that this claim holds true. Note that $b^\st$ can be also merged into the $1/\Tr({\bf A}{\bf R}^\st)$.  
	
	\begin{lemma} \label{lem:InvTr}
		The inverse function of the trace of the multiplication of a $n \times n$ positive semi-definite matrix ${\bf Z} \succeq 0$ and a deterministic positive semi-definite matrix ${\bf A}$, i.e., $f:S^{+}({\mathbb R}^n) \to {\mathbb R}, f({\bf Z})= 1/\Tr({\bf A}{\bf Z})$ is convex.
	\end{lemma}
	
	\begin{proof}
		See Appendix D.
	\end{proof}
	Since it holds that $\big({\boldsymbol \rho}^\st\big)^* \big({\boldsymbol \rho}^\st\big)^T \succeq 0$ and ${\bf R}^\st \succeq 0$, Lemma \ref{lem:InvTr} is applicable in our case. 
	\fi
	
	\section{Simulations}\label{sec:sim}
	In this section, we present simulation results to validate the performance of the correlated perturbation approach and compare it with non-perturbation and uncorrelated perturbation approaches. 
	\begin{figure}[t] 
		\centering
		\includegraphics[width=7.cm]{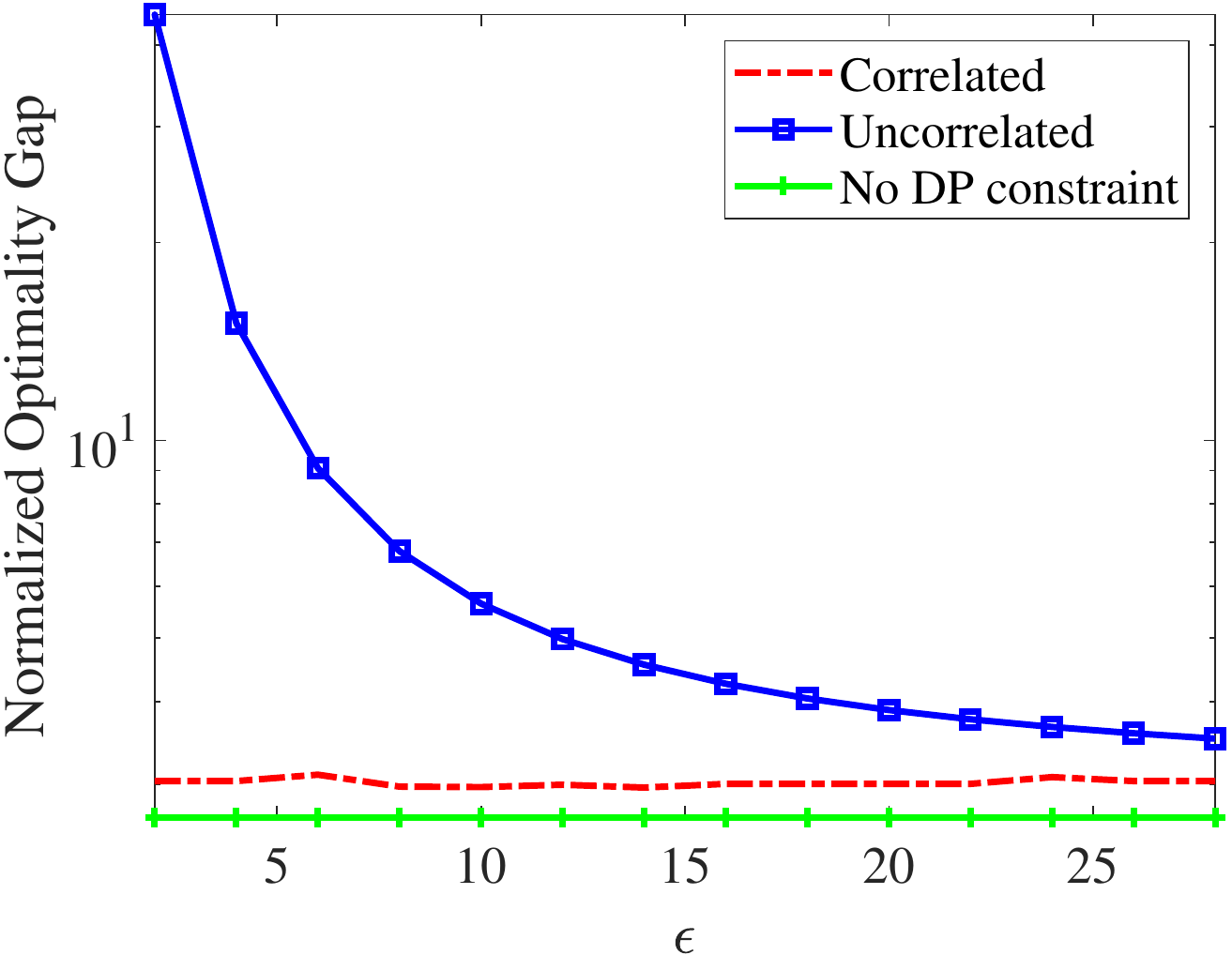}
		\caption{Optimality gap comparison under different privacy levels ($\text{SNR}=10dB, \delta=0.01$).}\label{pcn:optGapE}  
	\end{figure}
	
	\begin{figure}[t] 
		\centering
		\includegraphics[width=7.cm]{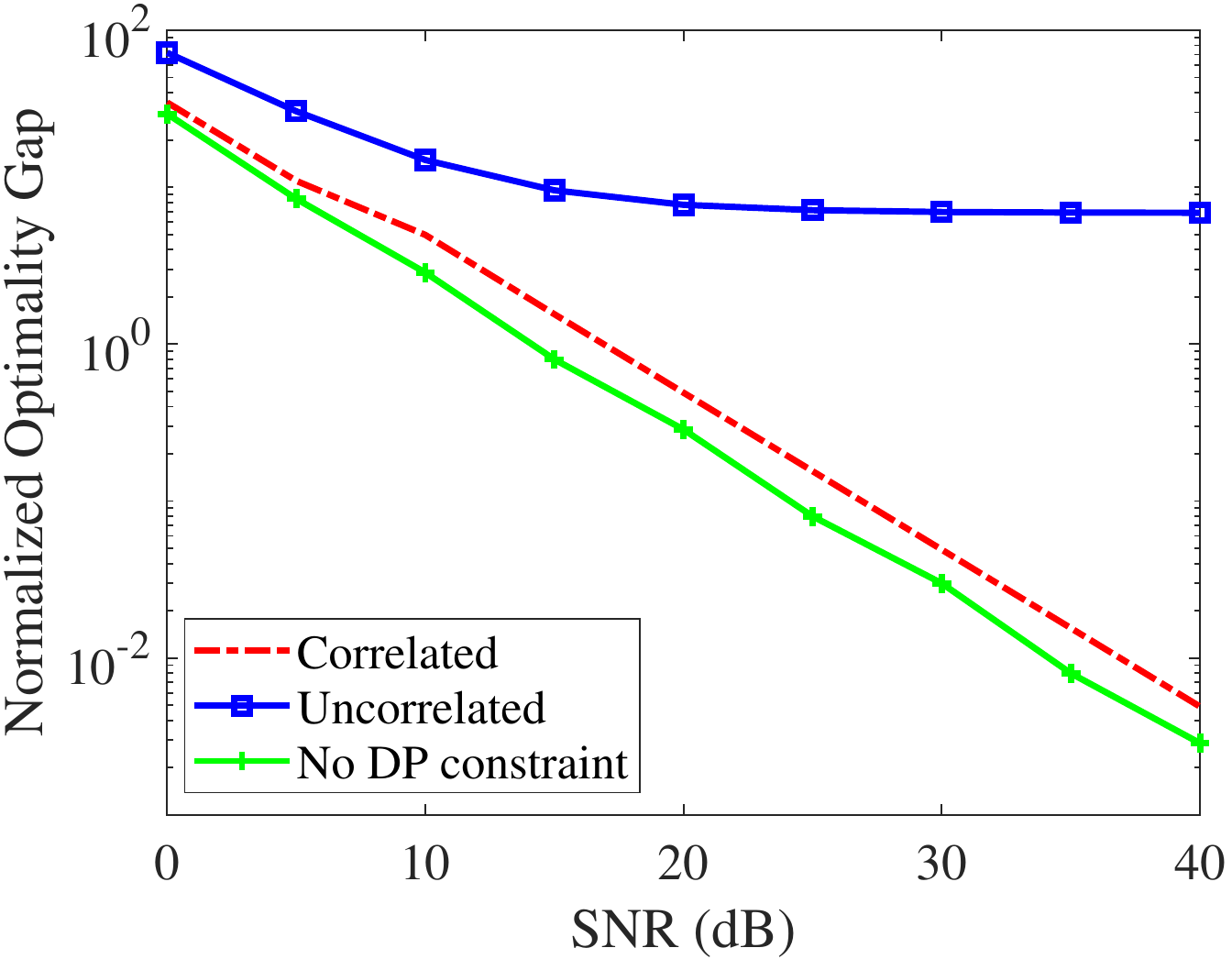}
		\caption{Optimality gap versus SNR for different perturbation approaches ($\epsilon =5, \delta=0.01$).}\label{pcn:optGapSNR}  
	\end{figure}
	
	\begin{figure}[t] 
		\centering
		\includegraphics[width=7.cm]{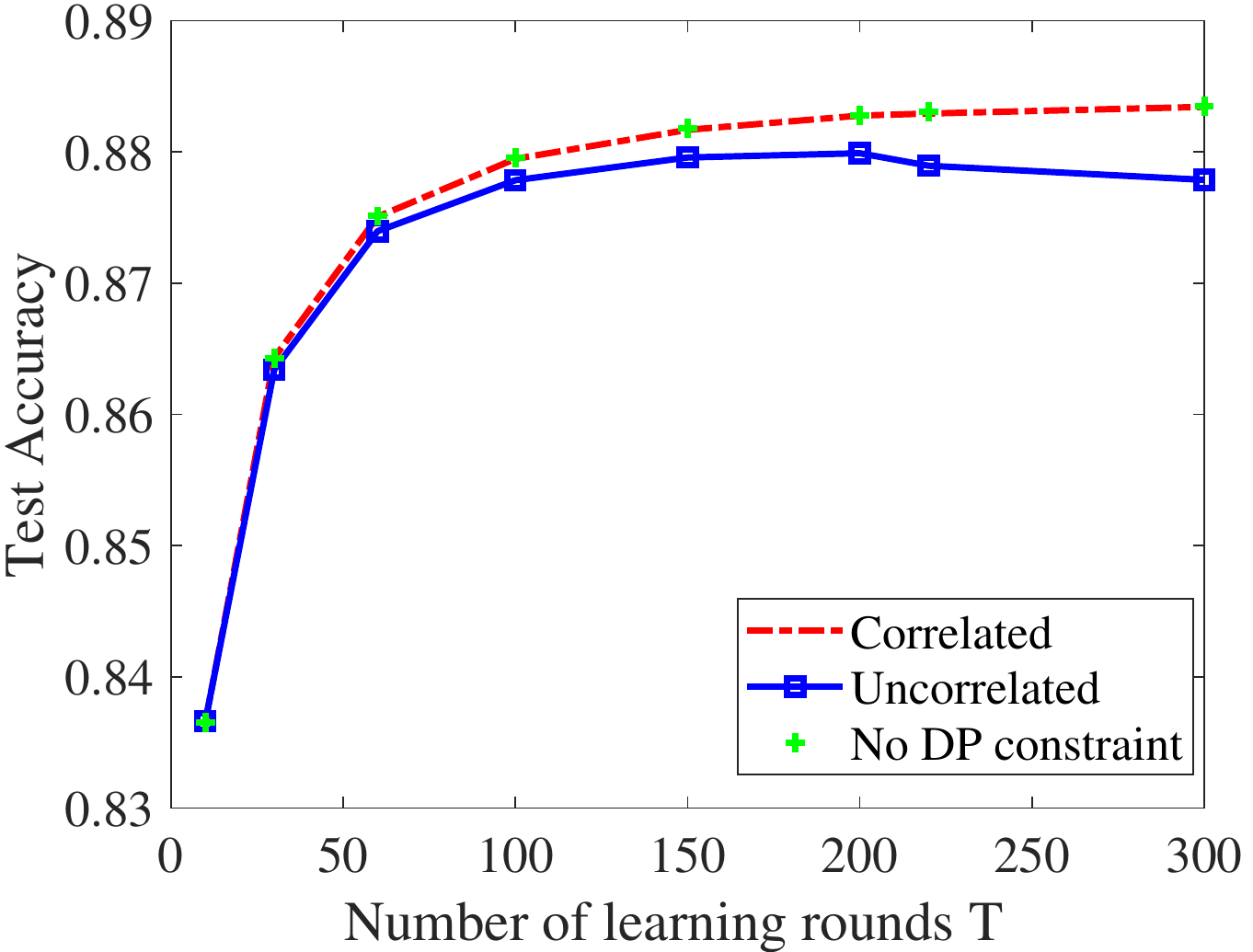}
		\caption{Test accuracy comparison with MNIST dataset. ($\text{SNR}=5dB, \epsilon=5, \delta=0.01$).}\label{pcn:TestAccu}  
	\end{figure}
	The test dataset contains $D_\text{tot}=10000$ samples, with model size $d=10$, data points ${\boldsymbol \mu} \overset{i.i.d.}{\sim}{\cal N}(0, {\bf I}_d)$, and labels $\nu={\boldsymbol \mu}(2)+3{\boldsymbol \mu}(5)+0.2z_o$ where $z_o \overset{i.i.d.}{\sim} {\cal N}(0,1)$ are the observation noises. We distribute the dataset evenly across $K=10$ users. The loss function is $f(\vecw, {\boldsymbol \mu}, \nu)=\frac{1}{2} \lVert \vecw^T{\boldsymbol \mu}-\nu \rVert^2+\zeta \lVert \vecw \rVert^2$ with $\zeta=0.5 \times 10^{-4}$. Parameters $\mu$ and $L$, are computed as the smallest and largest eigenvalues of data Gramian matrix $\Xi \triangleq {\bf U}^T{\bf U}+2D_\text{tot}\zeta{\bf I}$, where ${\bf U}=[{\boldsymbol \mu}_1,\dots,{\boldsymbol \mu}_{D_\text{tot}}]^T$ is a data matrix. The optimal solution is  ${\vecw^*}=\Xi^{-1}{\bf U}^T{\boldsymbol \nu}$, where ${\boldsymbol \nu}=[\nu_1,\dots,\nu_{D_\text{tot}}]^T$ is the label vector. The upper bounds of the local and global gradients are $\gamma^\st=2W \max_{({\boldsymbol \mu},\nu) \in {\cal D}} L({\boldsymbol \mu},\nu)$ and $G^\st_k=2WL_k$, where $W$ is an upper bound on $\Vert \vecw \Vert$; and $L({\boldsymbol \mu},\nu)$ and $L_k$ are the PL constants of $f(\vecw, {\boldsymbol \mu}, \nu)$ and $F_k(\vecw)$. We consider uniform privacy budget allocation in simulations, i.e., $\phi^\st=1/T, \forall t \in {\cal T}$.

	The wireless channel is modeled under Rice fading \cite{wang221federated},  
	and we set the line-of-sight (LoS) component to be 1. The channel coefficient can be expressed as
	\begin{equation*}
		h = \sqrt{\frac{\kappa }{1 + \kappa}} + \sqrt {\frac{1}{{1 + \kappa }}} \varrho, \end{equation*}
	where $\kappa$ is the Rician factor and $\varrho$ is the non-line-of-sight (NLoS) component obtained via  auto-regression: $\varrho^\st=\theta \varrho^{(t-1)}+\sqrt{1-\theta^2} \varphi^\st$. Here $\theta$ is the correlation coefficient and $\varphi^\st \overset{i.i.d.}{\sim}{\cal CN}(0,I)$ is an innovation process. The channel coefficients are given by $h_k=\varrho_s$ and $g_k=\varrho_a$, where the Rician factors are $\kappa_s=5$ and $\kappa_a=0$. The parameter $\theta$ is set to $0$ for simplicity since we assume perfect channel state information at the users. 
	
	Fig. \ref{pcn:optGapE} shows how the normalized optimality gap $[F(\vecw^{(T+1)})-F({\vecw^*})]/F({\vecw^*})$ varies with the DP parameter $\epsilon$. We set $\text{SNR}=10$dB, $\delta=0.01$, and we consider $T=30$ communication rounds. The results are averaged over 100 channel realizations. In the considered range of $\epsilon$, the correlated perturbation approach performs  approximately at the same level as the non-perturbation case and it is robust against different privacy levels. This shows that our proposed mechanism can guarantee both privacy and accuracy. In contrast, the uncorrelated perturbation approach shows an apparent compromise in convergence performance, especially with high privacy levels (smaller $\epsilon$). The same observation can be made in Fig. \ref{pcn:optGapSNR}, where we fix the DP level at $(5, 0.01)$ and then test the impacts of different SNR values on the optimality gap. Moreover, we observe a saturation trend for the uncorrelated perturbation approach in high SNR regime while this issue is resolved with our correlated perturbation approach.

	We then test our approach on the MNIST dataset via multinomial logical regression with cross-entropy loss function and quadratic regularization. 
	There are $D_\text{tot}=60000$ data samples composed of $C=10$ classes of handwritten digits. The original gradient data with dimension $d=784$, is pre-quantized into a manifold of lower dimension $30$ via principal component analysis (PCA) \cite{hein05intrinsic}. We set $\zeta=0.01$, $\Vert \vecw \Vert \le 10$, $\gamma^\st \le 50$, $\mu=0.3$ and $L=2.5$. In high privacy level and low SNR setting, e.g., $\epsilon=5, \delta=0.01$ and $\text{SNR}=5$dB, Fig. \ref{pcn:TestAccu} shows the test accuracy versus the value of communication rounds. It can be observed that the correlated perturbation approach provides higher test accuracy than the uncorrelated perturbation approach. Moreover, it approaches the performance of the non-perturbation case which clearly cannot provide any privacy guarantee.

	\section{Conclusions}\label{sec:con} 
	In this paper, we proposed a privacy-preserving design for OtA FL using correlated perturbations in the uplink transmission of gradient updates from distributed users to an edge server. The correlated perturbations provide privacy protection against an adversary node who intends to overhear the transmitted gradient vectors. In the meantime, our proposed design does not significantly compromise the learning accuracy as the aggregated perturbations add up to zero at the edge server. Based on theoretical analysis and numerical results of the SNR/SINR of the received updates, DP privacy, and convergence performance, we validated that our correlated perturbation design in OtA FL provides a good balance between privacy and learning performance as compared to the traditional methods with uncorrelated perturbations.  
	
	\section*{Appendix}
	The proof can be obtained referring to [Lemma 1, \cite{Osvaldo21}] and [Theorem 3.20, \cite{dwork2014algorithmic}] by taking into account the phase shifts of the channels and the effective channel gains at the adversary. 
	
	We recall the DP loss given in \eqref{pcn:dploss}, and then leverage the statistics of the effective noise ${\bf r}^\st \sim {\cal CN}(0, (m^\st)^2 {\bf I_d})$ 
	such that the DP loss can be reformulated into 		  
	\begin{multline*} 
		\!\!\!\!{\cal L}_{\caD,\caD'}(\vecy_a) 
		\!\!=\!\!\sum_{t=1}^T \ln  \frac{\exp\big(\frac{-\lVert \vecy_a^\st- \sum_{k=1}^{K} \sqrt{\eta^\st} \rho_k^\st \greFD \rVert^2}{(m^\st)^2} \big)}{\exp\big(\frac{-\lVert \vecy_a^\st- \sum_{k=1}^{K} \sqrt{\eta^\st} \rho_k^\st \greFDd \rVert^2}{(m^\st)^2}\big)}\\
		\overset{(a)}{=}\sum_{t=1}^T\! \frac{\lVert {\bf r}^\st\!+\!{\bf v}^\st \rVert^2-\!\lVert {\bf r}^\st \rVert^2}{(m^\st)^2}  
		\!=\!\sum_{t=1}^T\! \frac{2 \Re \{{{\bf r}^\st}^H \!{\bf v}^\st\}\!+\!\lVert {\bf v}^\st \rVert^2}{(m^\st)^2}  
	\end{multline*}
	where $(a)$ is derived using the predefined difference vector of the received gradients given neighboring global datasets ${\bf v}^\st$. Substituting the DP loss to $(\epsilon,\delta)$-DP condition, we get	
	\begin{multline*}
		\mathrm P\left(\Bigg |\sum_{t=1}^T \frac{2 \Re \{{{\bf r}^\st}^H {\bf v}^\st\}+\lVert {\bf v}^\st \rVert^2}{(m^\st)^2} \Bigg| >  \epsilon \right) \\
		\le \mathrm P\left(\Bigg |\sum_{t=1}^T \frac{2 \Re \{{{\bf r}^\st}^H {\bf v}^\st\} }{(m^\st)^2} \Bigg| >  \epsilon - \sum_t^T  \frac{\lVert {\bf v}^\st \rVert^2}{(m^\st)^2}\right)\\
		\overset{(b)}{\le}  \frac{2\sqrt{2 \left(\sum_{t} \left(\frac{\Delta^\st}{m^\st}\right)^2\right)}}{\sqrt{2\pi} \bigg[\epsilon-\sum_{t} \left(\frac{\Delta^\st}{m^\st}\right)^2\bigg]}  
		\exp \left(\frac{-\bigg[\epsilon-\sum_{t} \left(\frac{\Delta^\st}{m^\st}\right)^2\bigg]^2}{4 \left(\sum_{t=1}^{T} \left(\frac{\Delta^\st}{m^\st}\right)^2\right)} \right).
	\end{multline*}
	We get $(b)$ utilizing $\Re \{{{\bf r}^\st}^H {\bf v}^\st\} \sim  {\cal CN}(0, {(\Delta^\st m^\st)}^2/2 {\bf I})$ and the inequality of Gaussian distribution $  x \sim {\cal N}(0, \sigma^2 {\bf I})$, i.e., $\mathrm P(x > s) \le \frac{1}{\sqrt{2\pi}\sigma} \int_{s}^{\infty} \frac{x}{s} \exp(\frac{-x^2}{2\sigma^2})dx=\frac{\sigma}{\sqrt{2\pi}s}\exp(\frac{-s^2}{2\sigma^2})$. 
	Let $q=\frac{\epsilon-\tau}{2\sqrt{\tau}}, \tau=\sum_{t=1}^{T} \left(\frac{\Delta^\st}{m^\st}\right)^2$, $(\epsilon,\delta)$-DP condition is 
	
	\be \label{DPya}
	\mathrm P(\lvert {\cal L}_{\caD,\caD'}(\vecy_a) \rvert > \epsilon) \le \frac{1}{q \sqrt{\pi}} \exp(-q^2) < \delta.
	\ee
	For briefness, we let $C(x)=\sqrt{\pi} x \exp(x^2)$. \eqref{DPya} simplifies to  
	
	\be \label{DPyar}
	\tau < \bigg(\sqrt{\epsilon+\left(C^{-1}\big(1/\delta \big)\right)^2}-C^{-1}\big(1/\delta \big) \bigg)^2.
	\ee
	The conclusion in \eqref{DPyaf} can be easily obtained from \eqref{DPyar}.

	\bibliographystyle{IEEEtran}	
	\bibliography{PrivacyCorrelatedNoiseRef}
	
\end{document}